\documentclass[12pt]{article}
\usepackage[latin1]{inputenc}
\usepackage{cmap}
\usepackage{lmodern}

\usepackage{amssymb, amsmath, amsthm}
\usepackage[a4paper,top=25mm,bottom=25mm,left=25mm,right=25mm]{geometry}
\usepackage{etex}

\usepackage{authblk} 
\usepackage{pifont}
\usepackage{graphicx}
\usepackage[usenames,dvipsnames,svgnames,table]{xcolor}
\usepackage[figuresright]{rotating}
\usepackage{xtab} 
\usepackage{longtable} 
\usepackage{multirow}
\usepackage{footnote}
\usepackage[stable]{footmisc}
\usepackage{chngpage} 
\usepackage{pdflscape} 
\usepackage[nottoc,notlot,notlof]{tocbibind}

\usepackage{pgfplots}
\pgfplotsset{compat=1.14}
\usepackage{setspace}

\usepackage{array}
\newcolumntype{K}[1]{>{\centering\arraybackslash$}p{#1}<{$}}

\makesavenoteenv{tabular}
\usepackage{tabularx}
\usepackage{booktabs}
\usepackage{threeparttable}
\usepackage[referable]{threeparttablex} 
\newcolumntype{R}{>{\raggedleft\arraybackslash}X}
\newcolumntype{L}{>{\raggedright\arraybackslash}X}
\newcolumntype{C}{>{\centering\arraybackslash}X}
\newcolumntype{A}{>{\columncolor{gray!25}}C}
\newcolumntype{a}{>{\columncolor{gray!25}}c}

\usepackage{dcolumn} 
\newcolumntype{.}{D{.}{.}{-1}}

\usepackage{tikz}
\usetikzlibrary{arrows}
\usepackage[semicolon]{natbib}
\usepackage[hyphens]{url}
\usepackage{hyperref} 
\hypersetup{
  colorlinks   = true,    
  urlcolor     = blue,    
  linkcolor    = blue,    
  citecolor    = red      
}
\usepackage{microtype}
\usepackage[justification=centerfirst]{caption}

\usepackage[labelformat=simple]{subcaption}

\DeclareCaptionLabelFormat{parenthesis}{(#2)}
\makeatletter
\renewcommand\p@subfigure{\arabic{figure}.}
\makeatother

\DeclareCaptionLabelFormat{parenthesis}{(#2)}
\makeatletter
\renewcommand\p@subtable{A.\arabic{table}.}
\makeatother

\usepackage{enumitem}

\setlist[itemize]{leftmargin=2.5\parindent}
\setlist[enumerate]{leftmargin=2.5\parindent}

\theoremstyle{plain}

\newtheorem{lemma}{Lemma}
\newtheorem{proposition}{Proposition}
\newtheorem{theorem}{Theorem}

\theoremstyle{definition}
\newtheorem{axiom}{Axiom}

\newtheorem{definition}{Definition}
\newtheorem{example}{Example}

\theoremstyle{remark}

\newtheorem{remark}{Remark}

\def\keywords{\vspace{.5em} 
{\noindent \textit{Keywords}: }}

\def\JEL{\vspace{.5em} 
{\noindent \textbf{\emph{JEL} classification number}: }}

\def\AMS{\vspace{.5em} 
{\noindent \textbf{\emph{MSC} class}: }}

\author{\href{https://sites.google.com/site/laszlocsato87/}{L\'aszl\'o Csat\'o}\thanks{~E-mail: \emph{laszlo.csato@uni-corvinus.hu}} }
\affil{Institute for Computer Science and Control, Hungarian Academy of Sciences (MTA SZTAKI) \\
Laboratory on Engineering and Management Intelligence, Research Group of Operations Research and Decision Systems}
\affil{Corvinus University of Budapest (BCE) \\
Department of Operations Research and Actuarial Sciences}
\affil{Budapest, Hungary}
\title{Axiomatizations of inconsistency indices for triads}
\date{\today}

\def\Dedication{
{\noindent \emph{If you cannot prove your theorem, keep shifting parts \\
of the conclusion to the assumptions, until you can.}\footnote{~Source: \url{https://en.wikipedia.org/wiki/Ennio_de_Giorgi}}
}

\vspace{0.25cm}
\noindent \small{(Ennio de Giorgi)}
\vspace{1cm} }

\begin{document}

\maketitle

\Dedication

\begin{abstract}
\noindent
Pairwise comparison matrices often exhibit inconsistency, therefore many indices have been suggested to measure their deviation from a consistent matrix. A set of axioms has been proposed recently that is required to be satisfied by any reasonable inconsistency index. This set seems to be not exhaustive as illustrated by an example, hence it is expanded by adding two new properties.
All axioms are considered on the set of triads, pairwise comparison matrices with three alternatives, which is the simplest case of inconsistency. We choose the logically independent properties and prove that they characterize, that is, uniquely determine the inconsistency ranking induced by most inconsistency indices that coincide on this restricted domain.
Since triads play a prominent role in a number of inconsistency indices, our results can also contribute to the measurement of inconsistency for pairwise comparison matrices with more than three alternatives.

\keywords{Pairwise comparisons; Analytic Hierarchy Process (AHP); inconsistency index; axiomatic approach; characterization}

\AMS{90B50, 91B08}

\JEL{C44}
\end{abstract}

\section{Introduction} \label{Sec1}

Pairwise comparisons play an important role in a number of decision analysis methods such as the Analytic Hierarchy Process (AHP) \citep{Saaty1977, Saaty1980}. They also naturally emerge in country \citep{Petroczy2019} and higher education \citep{CsatoToth2019} rankings, in voting systems \citep{CaklovicKurdija2017}, as well as in sport tournaments \citep{BozokiCsatoTemesi2016, ChaoKouLiPeng2018, Csato2013a, Csato2017c}.
Theoretically, an appropriate set of $n-1$ pairwise comparisons would be sufficient to derive a set of weights or to rank all alternatives. However, usually, more information is available in real-life situations. For example, the decision makers are asked further questions because it increases the robustness of the result. It is also clear that a round-robin tournament can be fairer than a knockout format as a loss does not lead to the elimination of a player.

Nonetheless, the knowledge of extra pairwise comparisons has a price. First, processing this additional information is time-consuming. Second, the set of comparisons may become \emph{inconsistent}: if alternative $A$ is better than $B$, and $B$ is better than $C$, then $C$ still might turn out to be preferred over $A$.
While consistent preferences do not automatically imply the rationality of the decision maker, it is plausible to assume that strongly inconsistent preferences indicate a problem. Perhaps the decision maker has not understood the elicitation phase, or the strength of players varies during the tournament.

Thus it is necessary to measure the deviation from consistency. The first concept of inconsistency has probably been presented in \citet{KendallSmith1940}. Since then, several inconsistency indices have been proposed \citep{Saaty1977, Koczkodaj1993, DuszakKoczkodaj1994, Barzilai1998, AguaronMoreno-Jimenez2003, PelaezLamata2003, FedrizziFerrari2017}, and compared with each other \citep{BozokiRapcsak2008, BrunelliCanalFedrizzi2013, BrunelliFedrizzi2019, Cavallo2019}. \citet{Brunelli2018} offers a comprehensive overview of inconsistency indices and their ramifications.

Recently, some authors have applied an axiomatic approach by suggesting reasonable properties required from an inconsistency index \citep{BrunelliFedrizzi2011, Brunelli2016a, Brunelli2017, BrunelliFedrizzi2015, CavalloDApuzzo2012, KoczkodajSzwarc2014, KoczkodajUrban2018}.
There is also one \emph{characterization} in this topic: \citet{Csato2018a} introduces six independent axioms that uniquely determine the Koczkodaj inconsistency ranking induced by the Koczkodaj inconsistency index \citep{Koczkodaj1993, DuszakKoczkodaj1994}. In the case of such characterizations, the appropriate motivation of the properties is not crucial. The result only says that there remains a single choice \emph{if} one accepts all axioms.  

This work aims to connect these two research directions by placing the axioms of \citet{Brunelli2017} -- which is itself an extended set of the properties proposed by \citet{BrunelliFedrizzi2015} -- and \citet{Csato2018a} into a single framework. They will be considered on the domain of triads, that is, pairwise comparison matrices with only three alternatives.
\citet{BozokiRapcsak2008} have already proved that there exists a differentiable one-to-one correspondence between the inconsistency indices of Saaty \citep{Saaty1977} and Koczkodaj \citep{Koczkodaj1993, DuszakKoczkodaj1994} on this set, furthermore, almost all inconsistency indices are functionally dependent for triads \citep{Cavallo2019}.
We will show that the inconsistency ranking induced by this so-called \emph{natural triad inconsistency index} is the unique inconsistency ranking satisfying all properties on the set of triads.
Since triads play a prominent role in a number of inconsistency indices, our results can also contribute to the measurement of inconsistency for pairwise comparison matrices with more than three alternatives.

The paper is structured as follows. Section~\ref{Sec2} presents the setting and the properties of inconsistency indices suggested by \citet{Brunelli2017}. This axiomatic system is revealed  in Section~\ref{Sec3} to be not exhaustive. Section~\ref{Sec4} introduces two new axioms and discusses logical independence. The natural triad inconsistency ranking is characterized in Section~\ref{Sec5}. Finally, Section~\ref{Sec6} summarizes our results.

\section{Preliminaries} \label{Sec2}

A matrix $\mathbf{A} = \left[ a_{ij} \right] \in \mathbb{R}^{n \times n}$ is called a \emph{pairwise comparison matrix} if $a_{ij} > 0$ and $a_{ij} = 1 / a_{ij}$ for all $1 \leq i,j \leq n$.
A pairwise comparison matrix $\mathbf{A}$ is said to be \emph{consistent} if $a_{ik} = a_{ij} a_{jk}$ for all $1 \leq i,j,k \leq n$.

Let $\mathcal{A}$ denote the set of pairwise comparison matrices.
\emph{Inconsistency index} $I: \mathcal{A} \to \mathbb{R}$ associates a value for each pairwise comparison matrix.

\citet{BrunelliFedrizzi2015} have suggested and justified five axioms for inconsistency indices. They
are briefly recalled here.

\begin{axiom} \label{Axiom1}
\emph{Existence of a unique element representing consistency} ($URS$):
An inconsistency index $I: \mathcal{A} \to \mathbb{R}$ satisfies axiom $URS$ if there exists a unique $v \in \mathbb{R}$ such that $I(\mathbf{A}) = v$ if and only if $\mathbf{A} \in \mathcal{A}$ is consistent.
\end{axiom}

\begin{axiom} \label{Axiom2}
\emph{Invariance under permutation of alternatives} ($IPA$):
Let $\mathbf{A} \in \mathcal{A}$ be any pairwise comparison and $\mathbf{P}$ be any permutation matrix on the set of alternatives considered in $\mathbf{A}$.
An inconsistency index $I: \mathcal{A} \to \mathbb{R}$ satisfies axiom $IPA$ if $I(\mathbf{A}) = I(\mathbf{P} \mathbf{A} \mathbf{P}^\top)$.
\end{axiom}

\begin{axiom} \label{Axiom3}
\emph{Monotonicity under reciprocity-preserving mapping} ($MRP$):
Let $\mathbf{A} = \left[ a_{ij} \right] \in \mathcal{A}$ be any pairwise comparison matrix, $b \in \mathbb{R}$ and $\mathbf{A}(b) = \left[ a_{ij}^b \right] \in \mathcal{A}$.
An inconsistency index $I: \mathcal{A} \to \mathbb{R}$ satisfies axiom $MRP$ if $I(\mathbf{A}) \leq I \left( \mathbf{A}(b) \right)$ if and only if $b \geq 1$.
\end{axiom}

\begin{axiom} \label{Axiom4}
\emph{Monotonicity on single comparisons} ($MSC$):
Let $\mathbf{A} \in \mathcal{A}$ be any consistent pairwise comparison matrix, $a_{ij} \neq 1$ be a non-diagonal element and $\delta \in \mathbb{R}$.
Let $\mathbf{A}_{ij}(\delta) \in \mathcal{A}$ be the inconsistent pairwise comparison matrix obtained from $\mathbf{A}$ by replacing the entry $a_{ij}$ with $a_{ij}^\delta$ and $a_{ji}$ with $a_{ji}^\delta$.
An inconsistency index $I: \mathcal{A} \to \mathbb{R}$ satisfies axiom $MSC$ if
\begin{eqnarray*}
1 < \delta < \delta' & \Rightarrow & I(\mathbf{A}) \leq I \left( \mathbf{A}_{ij}(\delta) \right) \leq I \left( \mathbf{A}_{ij}(\delta') \right); \\
\delta' < \delta < 1  & \Rightarrow & I(\mathbf{A}) \leq I \left( \mathbf{A}_{ij}(\delta) \right) \leq I \left( \mathbf{A}_{ij}(\delta') \right).
\end{eqnarray*}
\end{axiom}

\begin{axiom} \label{Axiom5}
\emph{Continuity} ($CON$):
Let $\mathbf{A} = \left[ a_{ij} \right] \in \mathcal{A}$ be any pairwise comparison matrix.
An inconsistency index $I: \mathcal{A} \to \mathbb{R}$ satisfies axiom $CON$ if it is a continuous function of the entries $a_{ij}$ of $\mathbf{A} \in \mathcal{A}$.
\end{axiom}

\citet{Brunelli2017} has introduced a further reasonable property.

\begin{axiom} \label{Axiom6}
\emph{Invariance under inversion of preferences} ($IIP$):
Let $\mathbf{A} \in \mathcal{A}$ be any pairwise comparison matrix.
An inconsistency index $I: \mathcal{A} \to \mathbb{R}$ satisfies axiom $IIP$ if $I(\mathbf{A}) = I \left( \mathbf{A}^\top \right)$.
\end{axiom}

The six properties above do not contradict each other and none of them are superfluous.

\begin{proposition} \label{Prop21}
Axioms $URS$, $IPA$, $MRP$, $MSC$, $CON$, and $IIP$ are independent and form a logically consistent axiomatic system. 
\end{proposition}

\begin{proof}
See \citet[Theorem~1]{Brunelli2017}.
\end{proof}

A \emph{triad} is a pairwise comparison matrix with three alternatives, the smallest pairwise comparison matrix which can be inconsistent. Therefore, triads play a prominent role in the measurement of inconsistency. For instance, the Koczkodaj inconsistency index \citep{Koczkodaj1993, DuszakKoczkodaj1994}, the Pel\'aez-Lamata inconsistency index \citep{PelaezLamata2003}, and the family of inconsistency indices proposed by \citet{KulakowskiSzybowski2014} are all based on triads.

In this paper, we will focus on the set of triads $\mathcal{T}$, and inconsistency will be measured by a \emph{triad inconsistency index} $I: \mathcal{T} \to \mathbb{R}$.
Note that a triad $\mathbf{T} \in \mathcal{T}$ can be described by its three elements above the diagonal such that $\mathbf{T} = (t_{12}; \,t_{13}; \,t_{23})$ and $\mathbf{T}$ is consistent if and only if $t_{13} = t_{12} t_{23}$.

\section{Motivation} \label{Sec3}

The axiomatic system suggested by \citet{Brunelli2017} is not guaranteed to be exhaustive in the sense that it may allow for some strange inconsistency indices. Consider the following one.

\begin{definition} \label{Def31}
\emph{Scale-dependent triad inconsistency index}:
Let $\mathbf{T} = \left[ t_{ij} \right] \in \mathcal{T}$ be any triad. Its inconsistency according to the \emph{scale-dependent triad inconsistency index} $I^{SD}$ is
\[
I^{SD}(\mathbf{T}) = \left| t_{13} - t_{12} t_{23} \right| + \left| \frac{1}{t_{13}} - \frac{1}{t_{12} t_{23}} \right| + \left| t_{12} - \frac{t_{13}}{t_{23}} \right| + \left| \frac{1}{t_{12}} - \frac{t_{23}}{t_{13}} \right| + \left| t_{23} - \frac{t_{13}}{t_{12}} \right| + \left| \frac{1}{t_{23}} - \frac{t_{12}}{t_{13}} \right|.
\]
\end{definition}

The scale-dependent triad inconsistency index sums the differences of all non-diagonal matrix elements from the value exhibiting consistency.

\begin{proposition} \label{Prop31}
The scale-dependent triad inconsistency index $I^{SD}$ satisfies axioms $URS$, $IPA$, $MRP$, $MSC$, $CON$, and $IIP$.
\end{proposition}

\begin{proof}
It is straightforward to show that $I^{SD}$ satisfies $URS$, $IPA$, $CON$, and $IIP$.

Consider $MRP$. Due to the properties $IIP$ and $IPA$, it is enough to show that $\left| t_{13}^b - t_{12}^b t_{23}^b \right| \geq \left| t_{13} - t_{12} t_{23} \right|$ for every possible (positive) value of $t_{12}$, $t_{13}$, and $t_{23}$ if and only if $b \geq 1$. It can be assumed without loss of generality that $t_{13} - t_{12} t_{23} \geq 0$, which implies $t_{13}^b - t_{12}^b t_{23}^b \geq 0$.
Let $f(b) = t_{13}^b - t_{12}^b t_{23}^b$, so
\[
\frac{\partial f(b)}{\partial b} = \ln(b) \left( t_{13}^b - t_{12}^b t_{23}^b \right),
\]
in other words, $f(b)$ is a monotonically increasing (decreasing) function for $b \geq 1$ ($b \leq 1$).

Consider $MSC$. It can be assumed that $t_{13}$ is the entry to be changed because of the axiom $IPA$. $I^{SD} (\mathbf{T}) = 0$ if $t_{13} = t_{12} t_{23}$, and all terms in the formula of $I^{SD} \left( \mathbf{T}_{ij}(\delta) \right)$ increase gradually when $\delta$ goes away from $1$.
\end{proof}

According to the example below, the scale-dependent triad inconsistency index $I^{SD}$ may lead to questionable conclusions.

\begin{example} \label{Examp31}
Take two alternatives $A$ and $B$ such that the decision maker is indifferent between them. Assume that a third alternative $C$ appears in the comparison, and $A$ is judged three times better than $C$, while $B$ is assessed to be two times better than $C$. Suppose that $C$ is a divisible alternative and is exchanged by its half.

The two situations can be described by the triads:
\[
\mathbf{S} = \left[
\begin{array}{K{1.5em} K{1.5em} K{1.5em}}
    1     & 1     & 3     \\
    1     & 1     & 2     \\
     1/3  &  1/2  & 1     \\
\end{array}
\right]
\qquad
\text{and}
\qquad
\mathbf{T} = \left[
\begin{array}{K{1.5em} K{1.5em} K{1.5em}}
    1     & 1     & 6     \\
    1     & 1     & 4     \\
     1/6  &  1/4  & 1     \\
\end{array}
\right].
\]
Here $I^{SD}(\mathbf{S}) = 19/6 \approx 3.167$ and $I^{SD}(\mathbf{T}) = 5$. In other words, the scale-dependent inconsistency index suggests that triad $\mathbf{S}$ is less inconsistent than triad $\mathbf{T}$, contrary to the underlying data as inconsistency is not expected to be influenced by the `amount' of alternative $C$.
\end{example}

Example~\ref{Examp31} clearly shows that the axioms of \citet{Brunelli2017} should be supplemented even on the set of triads.

\section{An improved axiomatic system} \label{Sec4}

We propose two new axioms of inconsistency indices for triads.

\begin{axiom} \label{Axiom7}
\emph{Homogeneous treatment of alternatives} ($HTA$):
Let $\mathbf{T} = (1; \,t_{13}; \,t_{23})$ and $\mathbf{T'} = (1; \,t_{13} / t_{23}; \,1)$ be any triad.
A triad inconsistency index $I: \mathcal{T} \to \mathbb{R}$ satisfies axiom $HTA$ if $I(\mathbf{T}) = I \left( \mathbf{T'} \right)$.
\end{axiom}

According to homogeneous treatment of alternatives, if the first and the second alternatives are equally important on their own, but they are also compared to a third alternative, then the inconsistency of the resulting triad should not be influenced by the relative importance of the new alternative.

\begin{axiom} \label{Axiom8}
\emph{Scale invariance} ($SI$):
Let $\mathbf{T} = (t_{12}; \,t_{13}; \,t_{23})$ and $\mathbf{T'} = (k t_{12}; \,k^2 t_{13}; \,k t_{23})$ be any triads such that $k > 0$.
A triad inconsistency index $I: \mathcal{T} \to \mathbb{R}$ satisfies axiom $SI$ if $I(\mathbf{T}) = I
\left( \mathbf{T'} \right)$.
\end{axiom}

Scale invariance implies that inconsistency is independent of the mathematical representation of the preferences. For example, consider the following pairwise comparisons: the first alternative is `moderately more important' than the second and the second alternative is `moderately more important' than the third. It makes sense to expect the level of inconsistency to be the same if `moderately more important' is coded by the numbers $2$, $3$, or $4$, and so on, even allowing for a change in the direction of the two preferences.
If the encoding is required to preserve consistency, one arrives at the property $SI$.

Note that Example~\ref{Examp31} shows the violation of $SI$ by the scale-dependent triad inconsistency index $I^{SD}$.

$HTA$ and $SI$ have been introduced in \citet{Csato2018a} for inconsistency rankings (and $HTA$ has been called \emph{homogeneous treatment of entities} there).

In order to understand the implications of the extended axiomatic system, the logical consistency and independence of the eight properties should be discussed.

For this purpose, let us introduce the natural triad inconsistency index.

\begin{definition} \label{Def41}
\emph{Natural triad inconsistency index}:
Let $\mathbf{A} = \left[ a_{ij} \right] \in \mathbb{R}^{3 \times 3}_+$ be a triad. Its inconsistency according to the \emph{natural triad inconsistency index} $I^T$ is
\[
I^T(\mathbf{A}) = \max_{i<j<k} \left\{ \frac{a_{ik}}{a_{ij} a_{jk}}; \, \frac{a_{ij} a_{jk}}{a_{ik}} \right\}.
\]
\end{definition}

On the domain of triads, most inconsistency indices induce the same inconsistency ranking as the natural triad inconsistency index because they are functionally related \citep{BozokiRapcsak2008, Cavallo2019}.

\begin{proposition} \label{Prop41}
Axioms $URS$, $IPA$, $MRP$, $MSC$, $CON$, $IIP$, $HTA$, and $SI$ form a logically consistent axiomatic system on the set of triads.
\end{proposition}

\begin{proof}
The Koczkodaj inconsistency index satisfies all properties.
See \citet[Proposition~1]{Brunelli2017} for the axioms $URS$, $IPA$, $MRP$, $MSC$, $CON$, and $IIP$.
Homogeneous treatment of alternatives and scale invariance immediately follow from \citet[Theorem~1]{Csato2018a}.
\end{proof}

However, some axioms can be implied by a conjoint application of the others.

\begin{lemma} \label{Lemma41}
Axioms $IIP$, $HTA$, and $SI$ imply $IPA$ on the set of triads.
\end{lemma}

\begin{proof}
Let $\mathbf{T} = (t_{12}; \,t_{13}; \,t_{23})$ be a triad, $\mathbf{P}$ be a permutation matrix and $\mathbf{S} = \mathbf{P} \mathbf{T} \mathbf{P}^\top = (s_{12}; \,s_{13}; \,s_{23})$.
Let $I: \mathcal{T} \to \mathbb{R}$ be a triad inconsistency index satisfying $IIP$, $HTA$, and $SI$.

Consider $\mathbf{T_1} = (1; \,t_{13} / t_{12}^2; \,t_{23} / t_{12})$ and $\mathbf{S_1} = (1; \,s_{13} / s_{12}^2; \,s_{23} / s_{12})$. Then $I(\mathbf{T_1}) = I(\mathbf{T})$ and $I (\mathbf{S_1}) = I(\mathbf{S})$ according to $SI$.

Consider $\mathbf{T_2} = \left(1; \,t_{13} / (t_{12} t_{23}); \,1 \right)$ and $\mathbf{S_2} = \left( 1; \,s_{13} / (s_{12} s_{23}); \,1 \right)$. $HTA$ leads to $I(\mathbf{T_2}) = I(\mathbf{T_1})$ and $I (\mathbf{S_2}) = I(\mathbf{S_1})$.

$t_{13} / (t_{12} t_{23}) \geq 1$ and $s_{13} / (s_{12} s_{23}) \geq 1$ can be assumed without loss of generality because of the property $IIP$.

To summarize, $I(\mathbf{T}) = I(\mathbf{T_1}) = I(\mathbf{T_2})$ and $I(\mathbf{S}) = I(\mathbf{S_1}) = I(\mathbf{S_2})$.

The natural triad inconsistency index $I^T$ satisfies $IPA$, so $t_{13} / (t_{12} t_{23}) = s_{13} / (s_{12} s_{23})$, hence $\mathbf{T_2} = \mathbf{S_2}$, that is, $I(\mathbf{T_2}) = I(\mathbf{S_2})$ and $I(\mathbf{T}) = I(\mathbf{S})$.
\end{proof}

\begin{lemma} \label{Lemma42}
Axioms $URS$, $MSC$, $IIP$, $HTA$, and $SI$ imply $MRP$ on the set of triads.
\end{lemma}

\begin{proof}
Let $\mathbf{T} = (t_{12}; \,t_{13}; \,t_{23})$ and $\mathbf{T}(b) = (t_{12}^b; \,t_{13}^b; \,t_{23}^b)$ be any triads.
Let $I: \mathcal{T} \to \mathbb{R}$ be a triad inconsistency index satisfying $URS$, $MSC$, $IIP$, $HTA$, and $SI$.

Consider $\mathbf{T_1} = (1; \,t_{13} / t_{12}^2; \,t_{23} / t_{12})$ and $\mathbf{T_1}(b) = (1; \,t_{13}^b / t_{12}^{2b}; \,t_{23}^b / t_{12}^b)$. Then $I(\mathbf{T_1}) = I(\mathbf{T})$ and $I \left( \mathbf{T_1}(b) \right) = I(\mathbf{T}(b))$ according to $SI$.

Consider $\mathbf{T_2} = \left(1; \,t_{13} / (t_{12} t_{23}); \,1 \right)$ and $\mathbf{T_2}(b) = \left( 1; \,t_{13}^b / (t_{12}^b t_{23}^b); \,1 \right)$. $HTA$ leads to $I(\mathbf{T_2}) = I(\mathbf{T_1})$ and $I \left( \mathbf{T_2}(b) \right) = I \left( \mathbf{T_1}(b) \right)$.

It can be assumed without loss of generality that $t_{13} / (t_{12} t_{23}) \geq 1$ because of $IIP$.

To summarize, $I(\mathbf{T}) = I(\mathbf{T_1}) = I(\mathbf{T_2})$ and $I \left( \mathbf{T}(b) \right) =  I \left( \mathbf{T_1}(b) \right) = I \left( \mathbf{T_2}(b) \right)$.

If $t_{13} / (t_{12} t_{23}) > 1$, then $\mathbf{T_2}$ differs only in one non-diagonal element from the consistent triad with all entries equal to $1$.
Therefore, $I(\mathbf{T_2}) \leq I \left( \mathbf{T_2}(b) \right)$ if and only if $b \geq 1$ because of the property $MSC$.
Otherwise, $\mathbf{T_2}$ is consistent, and $I(\mathbf{T}) = I \left( \mathbf{T}(b) \right) = I(\mathbf{T_2}) = v$ due to $URS$.
\end{proof}

There exists no further direct implication among the remaining six properties.

\begin{theorem} \label{Theo41}
Axioms $URS$, $MSC$, $CON$, $IIP$, $HTA$, and $SI$ are independent on the set of triads.
\end{theorem}

\begin{proof}
Independence of a given axiom can be shown by providing a triad inconsistency index that satisfies all axioms except the one at stake:
\begin{enumerate}[label=\fbox{\arabic*}]
\item
$URS$: The triad inconsistency index $I^1: \mathcal{T} \to \mathbb{R}$ such that $I^1(\mathbf{T}) = 0$ for all triads $\mathbf{T} \in \mathcal{T}$.
\item
$MSC$: The triad inconsistency index $I^2: \mathcal{T} \to \mathbb{R}$ such that
\[
I^2(\mathbf{T}) = - \max \left\{ \frac{t_{13}}{t_{12} t_{23}}; \, \frac{t_{12} t_{23}}{t_{13}} \right\}
\]
for all triads $\mathbf{T} \in \mathcal{T}$. $I^2$ can be called the inverse natural triad inconsistency index.
\item
$CON$: The triad inconsistency index $I^3: \mathcal{T} \to \mathbb{R}$ such that
\[
I^3(\mathbf{T}) = 
\left\{ \begin{array}{ll}
0 & \textrm{if $\mathbf{T}$ is consistent} \\
\max \left\{ t_{13} / \left( t_{12} t_{23} \right); \, \left( t_{12} t_{23} \right) / t_{13} \right\} + 1 & \textrm{otherwise}
\end{array} \right.
\]
for all triads $\mathbf{T} \in \mathcal{T}$. $I^3$ is essentially the index $I^T$, but it is not continuous in the environment of consistent matrices.
\item
$IIP$: The triad inconsistency index $I^4: \mathcal{T} \to \mathbb{R}$ such that
\[
I^4(\mathbf{T}) = \frac{t_{13}}{t_{12} t_{23}}
\]
for all triads $\mathbf{T} \in \mathcal{T}$. $I^4$ is essentially the natural triad inconsistency index $I^T$, but takes only the entries above the diagonal into account.
\item
$HTA$: The triad inconsistency index $I^5: \mathcal{T} \to \mathbb{R}$ such that
\begin{equation} \label{eq1}
I^5(\mathbf{T}) = \left( \frac{t_{12}}{t_{23}} + \frac{t_{23}}{t_{12}} \right) \cdot \left( \max \left\{ \frac{t_{13}}{t_{12} t_{23}}; \, \frac{t_{12} t_{23}}{t_{13}} \right\} - 1 \right)
\end{equation}
for all triads $\mathbf{T} \in \mathcal{T}$.
\item
$SI$: The triad inconsistency index $I^6: \mathcal{T} \to \mathbb{R}$ such that
\[
I^6(\mathbf{T}) = \left| t_{12} - \frac{t_{13}}{t_{23}} \right| + \left| \frac{1}{t_{12}} - \frac{t_{23}}{t_{13}} \right|
\]
for all triads $\mathbf{T} \in \mathcal{T}$.
\end{enumerate}

Proving that the triad inconsistency index $I^i$ satisfies all axioms except for the $i$th is straightforward if $1 \leq i \leq 4$, therefore left to the reader.

Consider the triad inconsistency index $I^5$. It is easy to see that this function is continuous, nonnegative and equals to zero if and only if a triad is consistent ($t_{13} = t_{12} t_{23}$), as well as it meets invariance under inversion of preferences and scale invariance. $I^5$ also satisfies monotonicity on single comparisons because the second term in formula \eqref{eq1} is essentially the natural triad inconsistency index, and the first term is increasing in both $t_{12}$ and $t_{23}$ ceteris paribus, while it is independent of $t_{13}$.
Finally, take the triads $\mathbf{T} = (1; \,8; \,4)$ and $\mathbf{T'} = (1; \,2; \,1)$, which lead to $I^5(\mathbf{T}) = 17/4 \neq 5/2 = I^5(\mathbf{T'})$, showing the violation of $HTA$.

Now look at the triad inconsistency index $I^6$. It is trivial to verify that $I^6$ satisfies $URS$, $MSC$, $CON$, and $IIP$. $HTA$ is also met as $I^6(\mathbf{T}) = I^6(\mathbf{T'})$ when $\mathbf{T} = (1; \,t_{13}; \,t_{23})$ and $\mathbf{T'} = (1; \,t_{13} / t_{23}; \,1)$.
Take the triads $\mathbf{T} = (1; \,8; \,4)$ and $\mathbf{T'} = (2; \,32; \,8)$, which result in $I^6(\mathbf{T}) = 3/2 \neq 9/4 = I^6(\mathbf{T'})$, presenting the violation of $SI$.
\end{proof}

To conclude, the axiomatic system consisting of $URS$, $MSC$, $CON$, $IIP$, $HTA$, and $SI$ satisfies logical consistency and independence on the set of triads $\mathcal{T}$.

\section{Characterization} \label{Sec5}

It still remains a question whether the extended set of properties is exhaustive on the set of triads $\mathcal{T}$ or not. We will show that the axioms are closely related to the natural triad inconsistency index: they mean that $I^T$ is the only appropriate index for measuring the inconsistency of triads.

\begin{theorem} \label{Theo51}
Let $\mathbf{S},\mathbf{T} \in \mathcal{T}$ be any triads and $I: \mathcal{T} \to \mathbb{R}$ be a triad inconsistency index satisfying $MSC$, $IIP$, $HTA$, and $SI$.
Then $I^T(\mathbf{S}) \geq I^T(\mathbf{T})$ implies $I(\mathbf{S}) \geq I(\mathbf{T})$.
\end{theorem}

\begin{proof}
Assume that $I^T(\mathbf{S}) \geq I^T(\mathbf{T})$.
The idea is to gradually simplify the comparison of the inconsistencies of the two triads by using the axioms that are satisfied by the arbitrary triad inconsistency index $I: \mathcal{T} \to \mathbb{R}$.

Consider the triads $\mathbf{S_1} = (1; \,s_{13} / s_{12}^2; \,s_{23} / s_{12})$ and $\mathbf{T_1} = (1; \,t_{13} / t_{12}^2; \,t_{23} / t_{12})$. Since the natural triad inconsistency index satisfies $SI$, it is guaranteed that $I^T(\mathbf{S}) = I^T(\mathbf{S_1})$ and $I^T(\mathbf{T}) = I^T(\mathbf{T_1})$.

Consider the triads $\mathbf{S_2} = (1; \,s_{13} / (s_{12} s_{23}); \,1)$ and $\mathbf{T_2} = (1; \,t_{13} / (t_{12} t_{23}); \,1)$. As the natural triad inconsistency index meets $HTA$, it is known that $I^T(\mathbf{S_1}) = I^T(\mathbf{S_2})$ and $I^T(\mathbf{T_1}) = I^T(\mathbf{T_2})$.

$s_{13} / (s_{12} s_{23}) \geq 1$ and $t_{13} / (t_{12} t_{23}) \geq 1$ can be assumed without loss of generality due to $IIP$. Consequently, $I^T(\mathbf{S_2}) = I^T(\mathbf{S_1}) = I^T(\mathbf{S}) \geq I^T(\mathbf{T}) = I^T(\mathbf{T_1}) = I^T(\mathbf{T_2})$, which means that $s_{13} / (s_{12} s_{23}) \geq t_{13} / (t_{12} t_{23}) \geq 1$.

Starting from this inequality and using the properties of the triad inconsistency index $I:\mathcal{T} \to \mathbb{R}$, $MSC$ leads to $I(\mathbf{S_2}) \geq I(\mathbf{T_2})$, $HTA$ results in $I(\mathbf{S_1}) = I(\mathbf{S_2}) \geq I(\mathbf{T_2}) = I(\mathbf{T_1})$, and $SI$ implies that $I(\mathbf{S}) = I(\mathbf{S_1}) = I(\mathbf{S_2}) \geq I(\mathbf{T_2}) = I(\mathbf{T_1}) = I(\mathbf{T})$, which completes the proof.
\end{proof}

\begin{remark} \label{Rem51}
As Theorem~\ref{Theo51} shows, axioms $MSC$, $IIP$, $HTA$, and $SI$ allow for some odd triad inconsistency indices, for example, the \emph{flat triad inconsistency index} $I^F: \mathcal{T} \to \mathbb{R}$ such that $I^F(\mathbf{T}) = 0$ for any triad $\mathbf{T} \in \mathcal{T}$.
By attaching properties $URS$ and $CON$, inconsistency index $I^F$ is excluded, but they still allow for a \emph{`discretised' natural triad inconsistency index} $I^{DT}: \mathcal{T} \to \mathbb{R}$ defined as
\[
I^{DT}(\mathbf{T}) = 
\left\{ \begin{array}{ll}
I^{T}(\mathbf{T}) = \max \left\{ t_{13} / \left( t_{12} t_{23} \right); \, \left( t_{12} t_{23} \right) / t_{13} \right\} & \textrm{if $I^{T}(\mathbf{T}) \leq 2$} \\
2 & \textrm{otherwise}
\end{array} \right.
\]
for any triad $\mathbf{T} \in \mathcal{T}$.
\end{remark}

The proof of Theorem~\ref{Theo51} does not work in the reverse direction of $I(\mathbf{S}) \geq I(\mathbf{T}) \Rightarrow I^T(\mathbf{S}) \geq I^T(\mathbf{T})$ because monotonicity on single comparisons has been introduced without strict inequalities by \citet{BrunelliFedrizzi2015}.

\begin{axiom} \label{Axiom9}
\emph{Strong monotonicity on single comparisons} ($SMSC$):
Let $\mathbf{A} \in \mathcal{A}^{n \times n}$ be any consistent pairwise comparison matrix, $a_{ij} \neq 1$ a non-diagonal element and $\delta \in \mathbb{R}$.
Let $\mathbf{A}_{ij}(\delta) \in \mathcal{A}^{n \times n}$ be the inconsistent pairwise comparison matrix obtained from $\mathbf{A}$ by replacing the entry $a_{ij}$ with $a_{ij}^\delta$ and $a_{ji}$ with $a_{ji}^\delta$.
An inconsistency index $I: \mathcal{R}^n \to \mathbb{R}$ satisfies axiom $SMSC$ if
\begin{eqnarray*}
1 < \delta < \delta' & \Rightarrow & I(\mathbf{A}) < I \left( \mathbf{A}_{ij}(\delta) \right) < I \left( \mathbf{A}_{ij}(\delta') \right); \\
\delta' < \delta < 1  & \Rightarrow & I(\mathbf{A}) < I \left( \mathbf{A}_{ij}(\delta) \right) < I \left( \mathbf{A}_{ij}(\delta') \right).
\end{eqnarray*}
\end{axiom}

With the introduction of $SMSC$, there is no need for all of the six axioms.

\begin{lemma} \label{Lemma51}
Axioms $SMSC$, $CON$, $HTA$, and $SI$ imply $URS$ on the set of triads.
\end{lemma}

\begin{proof}
Let $\mathbf{S} = (s_{12}; \,s_{13}; \,s_{23})$ and $\mathbf{T} = (t_{12}; \,t_{13}; \,t_{23})$ be any triads.
Let $I: \mathcal{T} \to \mathbb{R}$ be a triad inconsistency index satisfying $SMSC$, $CON$, $HTA$, and $SI$.

First, it is shown that $I(\mathbf{S}) = I(\mathbf{T})$ if triads $\mathbf{S}$ and $\mathbf{T}$ are consistent.
Consider the triads $\mathbf{S_1} = (1; \,s_{13} / s_{12}^2; \,s_{23} / s_{12})$ and $\mathbf{T_1} = (1; \,t_{13} / t_{12}^2; \,t_{23} / t_{12})$. Then $I(\mathbf{S}) = I(\mathbf{S_1})$ and $I(\mathbf{T}) = I(\mathbf{T_1})$ due to $SI$.
Consider the triads $\mathbf{S_2} = (1; \,s_{13} / (s_{12} s_{23}); \,1)$ and $\mathbf{T_2} = (1; \,t_{13} / (t_{12} t_{23}); \,1)$. Then $I(\mathbf{S_1}) = I(\mathbf{S_2})$ and $I(\mathbf{T_1}) = I(\mathbf{T_2})$ because of $HTA$.
Furthermore, $\mathbf{S_2} = \mathbf{T_2}$, so $I(\mathbf{S}) = I(\mathbf{T})$.

Second, it is proved that $I(\mathbf{S}) \neq I(\mathbf{T})$ if triad $\mathbf{S}$ is consistent but $\mathbf{T}$ is inconsistent.
Consider the triads $\mathbf{S_1} = (1; \,s_{13} / s_{12}^2; \,s_{23} / s_{12})$ and $\mathbf{T_1} = (1; \,t_{13} / t_{12}^2; \,t_{23} / t_{12})$. Then $I(\mathbf{S}) = I(\mathbf{S_1})$ and $I(\mathbf{S}) = I(\mathbf{S_1})$ due to $SI$.
Consider the triads $\mathbf{S_2} = (1; \,s_{13} / (s_{12} s_{23}); \,1$ and $\mathbf{T_2} = (1; \,t_{13} / (t_{12} t_{23}); \,1)$. Then $I(\mathbf{S_1}) = I(\mathbf{S_2})$ and $I(\mathbf{T_1}) = I(\mathbf{T_2})$ because of $HTA$.
Furthermore, $s_{13} / (s_{12} s_{23}) = 1$ and $t_{13} / (t_{12} t_{23}) \neq 1$. Let $\delta \in \mathbb{R}$ and $\mathbf{T}_{ij}(\delta) \in \mathcal{T}$ be the inconsistent triad obtained from $\mathbf{T_2}$ by replacing the entry $t_{13} / (t_{12} t_{23})$ with $\left[ t_{13} / (t_{12} t_{23}) \right]^\delta$.
Assume, for contradiction, that $I(\mathbf{T}) = I(\mathbf{S})$. Then $I \left( \mathbf{T}(\delta) \right) < I \left( \mathbf{T}(1/2) \right) < I(\mathbf{S})$ for any $0 < \delta < 1/2$ due to strong monotonicity on single comparisons, which contradicts to continuity because $\lim_{\delta \to 0} \mathbf{T}(\delta) = \mathbf{S}$.
\end{proof}

As Theorem~\ref{Theo41} has already revealed, the weaker property of $MSC$ cannot substitute $SMSC$ in the proof of Lemma~\ref{Lemma51}.

\begin{proposition} \label{Prop51}
Axioms $SMSC$, $CON$, $IIP$, $HTA$, and $SI$ form a logically consistent and independent axiomatic system on the set of triads $\mathcal{T}$.
\end{proposition}

\begin{proof}
For consistency, it is sufficient to check that the natural triad inconsistency index $I^T$ satisfies strong monotonicity on single comparisons.

For independence, see the proof of Theorem~\ref{Theo41}. The inconsistency indices $I^3$, $I^4$, $I^5$, and $I^6$ satisfy $SMSC$, too.
\end{proof}

With this strengthening of $MSC$, we are able to characterize the natural triad inconsistency index on the set of triads.

\begin{proposition} \label{Prop52}
Let $\mathbf{S},\mathbf{T} \in \mathcal{T}$ be two triads and $I:\mathcal{T} \to \mathbb{R}$ be a triad inconsistency index satisfying $SMSC$, $IIP$, $HTA$, and $SI$.
Then $I(\mathbf{S}) \geq I(\mathbf{T})$ if and only if $I^T(\mathbf{S}) \geq I^T(\mathbf{T})$.
\end{proposition}

\begin{proof}
For the direction $I^T(\mathbf{S}) \geq I^T(\mathbf{T}) \Rightarrow I(\mathbf{S}) \geq I(\mathbf{T})$, see Theorem~\ref{Theo51}.

For $I(\mathbf{S}) \geq I(\mathbf{T}) \Rightarrow I^T(\mathbf{S}) \geq I^T(\mathbf{T})$, the proof of Theorem~\ref{Theo51} can be followed in the reverse direction with the assumption $I(\mathbf{S}) \geq I(\mathbf{T})$. The key point is the implication $I(\mathbf{S_2}) \geq I(\mathbf{T_2}) \Rightarrow s_{13} / (s_{12} s_{23}) \geq t_{13} / (t_{12} t_{23}) \geq 1$, which is guaranteed if the triad inconsistency index $I$ satisfies strong monotonicity on single comparisons, but not necessarily true if it meets only $MSC$.
\end{proof}

On the basis of Proposition~\ref{Prop52}, our main result can be formulated.

\begin{theorem} \label{Theo52}
The natural triad inconsistency index is essentially the unique triad inconsistency index satisfying strong monotonicity on single comparisons, invariance under inversion of preferences, homogeneous treatment of alternatives, and scale invariance.
\end{theorem}

The term \emph{essentially} refers to the fact that the four axioms $SMSC$, $IIP$, $HTA$, and $SI$ characterize only the \emph{inconsistency ranking} induced by the natural triad inconsistency index. Nonetheless, \citet{Csato2018a} argues that it does not make sense to distinguish inconsistency indices which rank pairwise comparison matrices uniformly.
Naturally, continuity can also be attached to these four axioms but it is rather a technical property. 

\begin{remark} \label{Rem52}
Remark~\ref{Rem51} remains valid in the case of \citet[Theorem~1]{Csato2018a}, which is true only in the following revised form:

\vspace{0.25cm}
\noindent
\emph{Let $\mathbf{A}$ and $\mathbf{B}$ two pairwise comparison matrices.
If $\succeq$ is an inconsistency ranking satisfying positive responsiveness, invariance under inversion of preferences, homogeneous treatment of entities, scale invariance, monotonicity, and reducibility, then $\mathbf{A} \succeq^K \mathbf{B}$ implies $\mathbf{A} \succeq \mathbf{B}$.}
\vspace{0.25cm}

Contrary to \citet[Theorem~1]{Csato2018a}, the implication does not hold in the other direction.
This problem can be easily solved by introducing the first axiom, positive responsiveness ($PR$) in a more powerful version called \emph{strong positive responsiveness} ($SPR$) with strict inequalities:

\vspace{0.25cm}
\noindent
\emph{Consider two triads $\mathbf{S} = (1; \,s_2; \,1)$ and $\mathbf{T} = (1; \,t_2; \,1)$ such that $s_2,t_2 \geq 1$.
Inconsistency ranking $\succeq$ satisfies $SPR$ if $\mathbf{S} \succ \mathbf{T} \iff s_2 < t_2$.}
\vspace{0.25cm}

Then the Koczkodaj inconsistency ranking would be the unique inconsistency ranking satisfying strong positive responsiveness, invariance under inversion of preferences, homogeneous treatment of entities, scale invariance, monotonicity, and reducibility.
\end{remark}

\section{Conclusions} \label{Sec6}

Axiomatic discussion of inconsistency measurement seems to be fruitful. While it is a well-established research direction in the choice of an appropriate weighting method \citep{Fichtner1984, Fichtner1986, BarzilaiCookGolany1987, Barzilai1997, CookKress1988, Bryson1995, Csato2017b, Csato2018c, BozokiTsyganok2019, Csato2019a, CsatoPetroczy2019b}, formal studies of inconsistency indices has not been undertaken until recently \citep{BrunelliFedrizzi2015, Brunelli2017, BrunelliFedrizzi2019, KoczkodajSzwarc2014, KoczkodajUrban2018, Csato2018a}.

The contribution of this paper can be shortly summarized as a unification of the two axiomatic approaches.
The first aims to justify reasonable properties and analyse indices in their light \citep{BrunelliFedrizzi2015, Brunelli2017}.
The second concentrates on the exact derivation of certain indices without spending too much time on the motivation of the axioms \citep{Csato2018a}. In particular, the axiomatic system of \citet{Brunelli2017} has been presented to be not exhaustive even for only three alternatives. However, by the introduction of two new properties, a unique triad inconsistency ranking can be identified.

Although most inconsistency indices are functionally related on this domain \citep{Cavallo2019}, hence they induce the same inconsistency ranking, our main finding is a powerful argument against indices which violate some of the axioms on the set of triads, like the Ambiguity Index \citep{SaloHamalainen1995, SaloHamalainen1997}, the Relative Error \citep{Barzilai1998}, or the Cosine Consistency Index \citep{KouLin2014}. This fact illustrates that it is worth discussing inconsistency indices on special classes of pairwise comparison matrices, similarly to \citet{CernanovaKoczkodajSzybowski2018}.
The results derived here can serve as a solid basis for measuring the inconsistency of pairwise comparison matrices for order greater than three.

\section*{Acknowledgements}
\addcontentsline{toc}{section}{Acknowledgements}
\noindent
We would like to thank \emph{S\'andor Boz\'oki}, \emph{Matteo Brunelli} and \emph{Mikl\'os Pint\'er} for inspiration. \\
\emph{Tam\'as Halm} and two anonymous reviewers provided valuable comments and suggestions on an earlier draft. \\
The research was supported by OTKA grant K 111797 and by the MTA Premium Postdoctoral Research Program.

\bibliographystyle{apalike}
\bibliography{All_references}

\end{document}